\documentclass[10pt,twocolumn,letterpaper]{article}
\usepackage[margin=0.78in]{geometry}
\usepackage{times}
\usepackage{amsmath,amssymb,amsthm,amsfonts,bm}
\usepackage{enumitem}
\usepackage{graphicx}
\usepackage{booktabs}
\usepackage{algorithm}
\usepackage[noend]{algpseudocode}
\usepackage{caption}
\usepackage{subcaption}
\usepackage{microtype}
\usepackage{url}
\usepackage{xcolor}
\usepackage[colorlinks=true,linkcolor=black,citecolor=black,urlcolor=blue!55!black]{hyperref}
\usepackage{fancyhdr}

\theoremstyle{plain}
\newtheorem{proposition}{Proposition}

\theoremstyle{remark}
\newtheorem{remark}{Remark}

\newcommand{\nn}{\mathbf{n}}
\newcommand{\lv}{\bm{\ell}}

\title{\vspace{-6mm}\bf Geometry-Driven Shadow Harmonisation for Composited Faces:\\
A Multiplicative, Albedo-Preserving Relighting Pipeline}

\author{
\normalsize Vijesh KP\\[1mm]
\normalsize \texttt{vijeshkpaei@gmail.com}
}
\date{}

\begin{document}
\maketitle
\thispagestyle{fancy}

\begin{abstract}
\noindent
Face swapping and face compositing pipelines routinely produce a face that is
geometrically well aligned but photometrically implausible: the donor face
carries flat, near-frontal studio illumination while the host body and
background carry directional scene light. Most existing remedies re-synthesise
the face---through colour transfer, neural relighting, or full inverse
rendering---and therefore risk altering identity, skin tone, and texture. We
present a deliberately conservative alternative: \emph{geometry-driven form-shadow
injection}. The pipeline never repaints the face. It estimates a per-pixel gain
field $g\in[g_{\min},1]$ from a rasterised 3D face proxy and multiplies it
channel-uniformly onto the original linear RGB, so the operator can only darken
and cannot shift chromaticity. A dense landmark mesh is rasterised into a depth
buffer, from which we derive surface normals, a depth-difference cavity term,
and screen-space cast shadows. Key-light \emph{direction} is estimated from
host-side cues (body, background, hair halo); on-face cues are downweighted
because they recover the donor's lighting. Shadow \emph{magnitude} is not matched
to the host: it is set by a three-parameter transfer $(\tau,\sigma,g_{\min})$.
The shading field is divided by its $75$th percentile over skin, so the
brightest quartile maps to unity, then gated, scaled, clamped, smoothed, and
re-clipped before compositing inside a feathered, skin-gated face mask. On an
analytic face heightfield, the default $(\tau,\sigma,g_{\min})=(0.90,0.45,0.82)$
modifies $56.5\%$ of face pixels with mean gain $0.938$ ($0.890$ on modified
pixels) and drives $3.4\%$ of pixels to the floor. Hue invariance is a corollary
of the operator, not an empirical finding. We analyse the transfer in closed
form, ablate its parameters, and discuss the failure modes of a monotone,
darkening-only formulation---including double-shadowing of non-flat donors. Implementation code is released at: \url{https://github.com/vijeshkpaei/geometry-driven-shadow-harmonisation}   
\end{abstract}

\vspace{1mm}

\section{Introduction}

A composited or swapped face is usually evaluated by two criteria: whether the
identity is preserved and whether the seam is invisible. A third criterion is
often the one that actually betrays the composite to a human observer---whether
the \emph{light} on the face agrees with the light on everything else in the
frame. Donor faces are frequently harvested from catalogue or passport-like
imagery captured under soft, frontal, near-shadowless illumination. When such a
face is transplanted onto a body photographed under a directional key light, the
face reads as flat: the nose casts no shadow, the eye sockets carry no
occlusion, and both cheeks are equally bright while the shoulder below shows an
unmistakable light direction.

The obvious response is to relight the face. Full relighting, however, is an
aggressive operation. Colour-transfer methods \cite{reinhard2001} match global
statistics and consequently move skin tone; neural relighting and portrait
harmonisation networks \cite{sun2019,zhou2019,pandey2021} synthesise new pixels
and can subtly alter identity, erase pores, or hallucinate specular highlights;
inverse-rendering approaches \cite{barron2015,sengupta2018} must separate albedo
from shading, and any error in that separation is written directly into the
output. In a production compositing context---where the face has already been
chosen, graded, and approved---these are unattractive risks.

This paper takes the opposite stance. We assume the donor face is a usable
stand-in for albedo---correct in colour and identity, missing only
\emph{form shadow}---and that this stand-in must not be re-synthesised. We do
not match host irradiance: we inject geometric darkening at a conservative
magnitude. The operator is a single, strictly non-expansive multiply,
\begin{equation}
  I_{\text{out}}(\mathbf{x}) \;=\; I_{\text{in}}(\mathbf{x})\, g(\mathbf{x}),
  \qquad g(\mathbf{x}) \in [g_{\min},\,1],
  \label{eq:core}
\end{equation}
with the same scalar $g$ applied to all three colour channels, after
linearisation from sRGB. Under the hypotheses of
Proposition~\ref{prop:bound}, three properties follow from the transfer
alone---independent of how badly the geometry or light \emph{direction}
fails---provided a post-smoothing clip restores $g\in[g_{\min},1]$ and the
optional tint path is left at zero:

\begin{enumerate}\itemsep1pt
  \item \textbf{No brightening.} Since $g \le 1$, the method cannot invent
        highlights or clip into the white point.
  \item \textbf{No chromaticity shift.} Since $g$ is channel-uniform, linear
        RGB ratios $(R\!:\!G\!:\!B)$ are invariant; only radiance scales.
  \item \textbf{Bounded darkening.} Since $g \ge g_{\min}$, the worst-case
        linear-radiance reduction is the known constant $1-g_{\min}$.
\end{enumerate}

These invariants bound photometric damage; they do not bound perceptual
error. A misplaced $18\%$ shadow is still an artefact, and a donor that is
already shaded can be double-shadowed.

\begin{proposition}[Bounded, chromaticity-preserving operator]
\label{prop:bound}
Let $I_{\mathrm{out}} = I_{\mathrm{in}}\, g$ with
$g:\Omega\to[g_{\min},1]$, $g_{\min}\in[0,1]$, applied identically to the
R, G, B channels at each pixel in linear RGB, and with the optional tint
path disabled. Then for every pixel $\mathbf{x}\in\Omega$:
\begin{enumerate}[label=(\roman*),itemsep=1pt,topsep=2pt]
  \item $I_{\mathrm{out}}(\mathbf{x}) \le I_{\mathrm{in}}(\mathbf{x})$
        componentwise (no brightening);
  \item if $I_{\mathrm{in},c'}(\mathbf{x})\ne 0$, the chromaticity ratio
        $I_{\mathrm{out},c}(\mathbf{x}) / I_{\mathrm{out},c'}(\mathbf{x})
        = I_{\mathrm{in},c}(\mathbf{x}) / I_{\mathrm{in},c'}(\mathbf{x})$
        for any two channels $c,c'$;
  \item $\lVert I_{\mathrm{out}}(\mathbf{x}) - I_{\mathrm{in}}(\mathbf{x})
        \rVert \le (1-g_{\min})\lVert I_{\mathrm{in}}(\mathbf{x})\rVert$
        (bounded deviation, any absolute vector norm).
\end{enumerate}
\end{proposition}
\begin{proof}
(i) and (iii) follow from $g\in[g_{\min},1]$ and
$I_{\mathrm{out}}-I_{\mathrm{in}}=I_{\mathrm{in}}(g-1)$. For (ii),
$I_{\mathrm{out},c}=g\,I_{\mathrm{in},c}$ cancels in any well-defined
channel ratio. Linear chromaticity is therefore exact. Encoded sRGB
channel ratios are only approximately invariant, because the IEC 61966
transfer has an affine offset; the discrepancy is negligible at
photographic luminances. Face-only compositing (Sec.~\ref{sec:comp})
replaces $g$ by the convex combination
$g_{\mathrm{eff}}=1-\alpha(1-g)\in[g_{\min},1]$, so (i)--(iii) survive
masking.
\end{proof}

\begin{remark}
Proposition~\ref{prop:bound} is a property of the applied gain, not of the
pipeline that produced it. It therefore supplies a \emph{worst-case
photometric bound} ($g\equiv g_{\min}$) without a case-by-case error
analysis. Section~\ref{sec:exp} does not re-measure that bound: it reports
typical gain statistics on one controlled proxy, under a known light
direction. The bound is also voided if $g$ is smoothed without a subsequent
clip, or if the tint path is enabled.
\end{remark}

\noindent The design question is then: how do we produce a $g$ field that is
geometrically plausible and stays inside the bound? Our contributions are:

\begin{itemize}\itemsep1pt
  \item A landmark-driven rasterisation stage that converts 468 face-mesh
        vertices into a dense depth buffer, normal field, coverage mask, and
        depth-derived cavity map at $512\times512$
        (Sec.~\ref{sec:geom}).
  \item A key-light estimator whose \emph{primary} cues are host-side (body
        silhouette, background highlights, hair/jaw halo). On-face cues are
        assigned the four lowest weights, because on a swapped face they
        recover the donor's studio lighting (Sec.~\ref{sec:light}).
  \item A shadow field combining wrapped Lambertian shading, screen-space
        cast shadows ray-marched on the depth buffer, and the cavity term,
        followed by $75$th-percentile normalisation that pins the brightest
        quartile of skin to unity (Sec.~\ref{sec:shadow}).
  \item A three-parameter shadow-to-gain transfer $(\tau,\sigma,g_{\min})$
        with a post-smoothing clip, analysed in Sec.~\ref{sec:gain} and
        characterised in Sec.~\ref{sec:exp}.
\end{itemize}

\begin{figure*}[t]
  \centering
  \includegraphics[width=\textwidth]{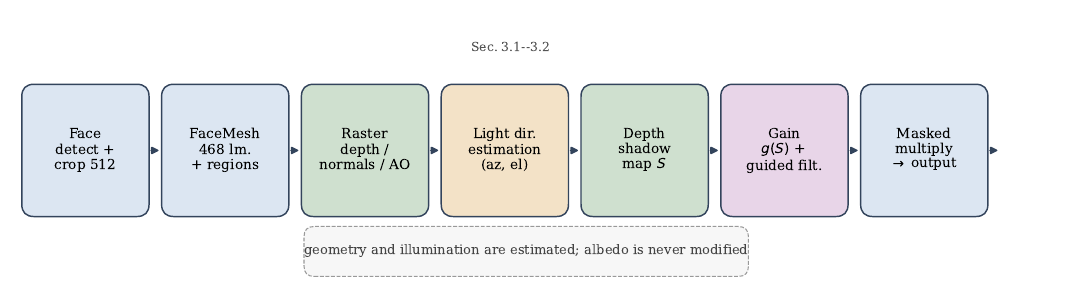}
  \caption{\textbf{Pipeline overview.} The image passes through detection,
  dense-landmark meshing, depth/normal/AO rasterisation, cue-based light
  estimation, shadow-map construction, gain transfer, and masked multiplication.
  Only the final stage touches pixels, and it touches them multiplicatively.
  Albedo, hair, clothing and background are untouched by construction.}
  \label{fig:pipeline}
\end{figure*}

\section{Related Work}

\paragraph{Image harmonisation.} Classical harmonisation matches low-order
colour statistics between a foreground and its new background
\cite{reinhard2001}, or solves a gradient-domain problem to hide the seam
\cite{perez2003}. Learned harmonisation \cite{tsai2017,cong2020} improves
realism but operates on the composite as an opaque appearance-transfer task; it
does not distinguish between an incorrect \emph{albedo} and an incorrect
\emph{shading}, which is precisely the distinction our problem demands.

\paragraph{Portrait relighting.} A large body of work relights faces from a
single image, either by regressing a spherical-harmonic illumination and
re-rendering \cite{zhou2019}, by training on light-stage captures
\cite{sun2019,pandey2021}, or by explicit inverse rendering into albedo,
normals, and light \cite{barron2015,sengupta2018}. These produce far richer
effects than we do---including specularity and subsurface transport---but they
regenerate skin pixels, and they provide no bound on identity or texture drift.

\paragraph{Geometry priors for faces.} Dense landmark regressors such as
MediaPipe Face Mesh \cite{kartynnik2019} return several hundred 3D vertices in
real time; 3D Morphable Models \cite{blanz1999} give a metrically better shape
at greater cost. Our depth proxy is intentionally coarse: because the output is
a bounded, smoothed, low-strength multiplier, a proxy that is correct in its
gross relief (nose forward, sockets back) suffices, and fine shape error is
attenuated by the $\sigma$ scaling and by smoothing $g$.

\paragraph{Screen-space shading.} Screen-space ambient occlusion
\cite{mittring2007} and screen-space ray-marched shadows \cite{sousa2011} are
standard real-time graphics techniques. We reuse the same \emph{depth-buffer}
setting on a face-only crop, where the usual weakness---missing off-screen
occluders---is largely irrelevant for nose, brow, and lip shadows, though it
still fails for a hat brim or a raised hand. Our cavity term is a
depth-unsharp-mask approximation, not a sampled hemisphere AO.

\section{Method}

\subsection{Overview and notation}

Let $I \in [0,1]^{H\times W\times 3}$ be the input photograph. We work in a
linearised space $\tilde I = \mathrm{srgb}^{-1}(I)$, since a multiplicative
darkening is only physically meaningful on linear radiance. Table~\ref{tab:not}
summarises notation. Figure~\ref{fig:pipeline} gives the stage sequence and
Fig.~\ref{fig:maps} shows every intermediate field on a synthetic subject.

\begin{table}[t]
\centering\small
\begin{tabular}{@{}ll@{}}
\toprule
Symbol & Meaning \\
\midrule
$D$ & rasterised face depth ($+$ toward camera) \\
$\nn$ & per-pixel unit surface normal \\
$C$ & mesh coverage ($1$ on face, $0$ elsewhere) \\
$A$ & cavity term (AO proxy), $A\in[0,1]$ \\
$\lv$ & unit direction \emph{toward} the key light \\
$\Lambda$ & wrapped Lambertian direct response \\
$V$ & screen-space cast-shadow visibility \\
$\rho$ & ambient ratio (soft vs.\ hard light) \\
$S$ & normalised shadow map, $S\in[0,1]$ \\
$g$ & per-pixel gain, $g\in[g_{\min},1]$ after the post-filter clip \\
$\alpha$ & face-only compositing mask \\
$k,\tau,\sigma$ & depth scale, lit-gate threshold, strength \\
\bottomrule
\end{tabular}
\caption{Notation.}
\label{tab:not}
\end{table}

\subsection{Face localisation and canonical crop}

A detector returns a face bounding box, which we expand by a crop scale
$s_c=2.2$ to a square region and resample to a canonical $512\times512$ working
crop. The expansion is deliberate: it admits the jaw, hair boundary, and a
margin of neck, all of which the light estimator (Sec.~\ref{sec:light})
consumes. The upsample to $512$ matters because catalogue faces are frequently
small in the source frame, and the mesh, depth buffer, and ray march all need
enough spatial resolution for the nose shadow to be more than a few pixels wide.

A dense mesh regressor then returns $468$ 3D vertices (plus iris refinements),
from which we fill semantic region masks: face oval, eyes, brows, lips, nose,
hair halo, neck, periocular ring, and outer rim. If detection or meshing fails,
the image is copied through unmodified. A \emph{wrong} mesh that still
returns vertices is not a fail: it will apply up to $1-g_{\min}$ of
darkening in the wrong places.

\subsection{Scene context}

Once per photograph we compute a person matte by portrait segmentation and
define the background as its complement, eroded away from the silhouette to
avoid matting fringe. This context is used \emph{only} to vote on light
direction. It is never blended, tinted, or propagated into the face.

\subsection{Geometry rasterisation}
\label{sec:geom}

We orient the mesh with $+X$ right, $+Y$ up, $+Z$ toward the camera, flipping
and scaling the regressor's $z$ by a depth scale $k$ (default $k=1.15$), which
mildly exaggerates relief and compensates for the well-known depth flattening of
landmark regressors. Mesh triangles are $z$-buffered into the crop, yielding
$D$, $C$, and, by central differences on $D$ in \emph{normalised} crop
coordinates (range $[-1,1]$),
\begin{equation}
  \nn \;=\; \frac{(-\partial_x D,\;\partial_y D,\;1)}
                 {\lVert(-\partial_x D,\;\partial_y D,\;1)\rVert}.
  \label{eq:normal}
\end{equation}
Here $x$ increases to the right and $y$ increases \emph{down} the image, so
the middle component equals $-\partial_Y D$ in the $+Y$-up camera frame.
(Using pixel coordinates without the normalisation would collapse $\nn$
toward $(0,0,1)$.) $\lv$ is the direction from the face toward the key
light, matching $\nn\cdot\lv$ in Eq.~\ref{eq:lambert}.

\paragraph{Cavity term.} A pixel that lies \emph{behind} its own
low-frequency neighbourhood sits in a socket, alar crease, or lip line.
With $\bar D_\varsigma = G_\varsigma * D$ a Gaussian-blurred depth
($\varsigma=14$~px) and
$\delta = (D-\bar D_\varsigma)/\mathrm{std}(D\mid C{>}0.5)$ standardised
over face pixels, we set
\begin{equation}
  A \;=\; 1 - a\bigl(1 - \mathrm{sigm}(\lambda\delta)\bigr),
  \label{eq:ao}
\end{equation}
with $a=0.75$, $\lambda=2$, and $\mathrm{sigm}(u)=(1+e^{-u})^{-1}$. This is
a signed depth-unsharp mask, not sampled-hemisphere AO: convexities
($\delta>0$) stay near $1$; concavities drop toward $1-a$. Landmark meshes
under-represent sharp creases, so we then apply a semantic cavity boost
\begin{equation}
  A \;\leftarrow\; \mathrm{clip}\bigl(A - a_f f (1-A+0.25),\,0,\,1\bigr),
\end{equation}
with $a_f=0.18$ and $f\in[0,1]$ a blob prior on eyes, lips, and nose wings.

\begin{figure*}[t]
  \centering
  \includegraphics[width=\textwidth]{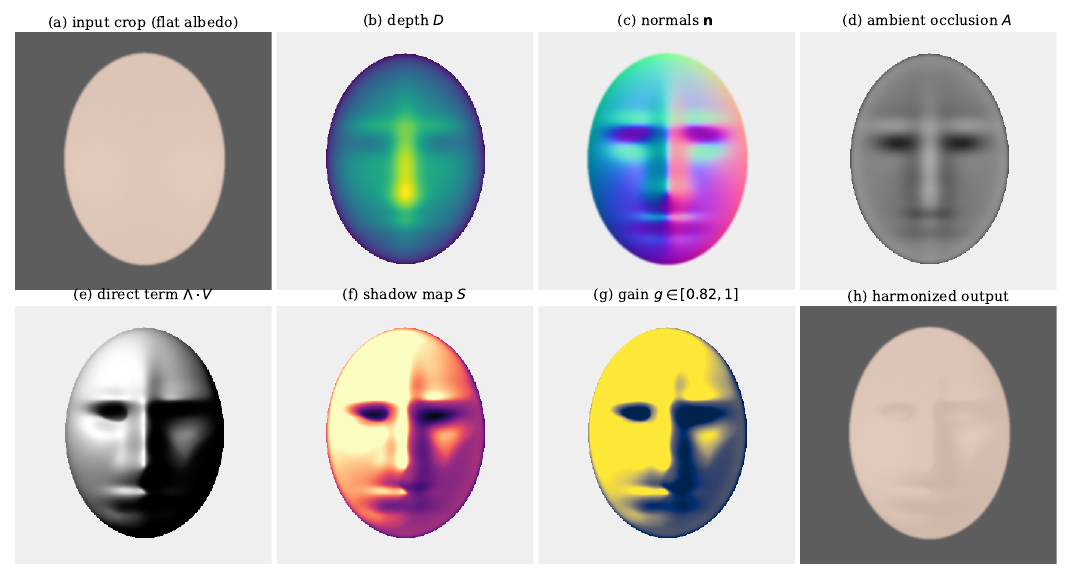}
  \caption{\textbf{Intermediate fields}, computed on an analytic face heightfield
  (Sec.~\ref{sec:exp}). (a) the flat input; (b) rasterised depth, nose forward
  and sockets recessed; (c) the normal field; (d) the depth cavity term $A$;
  (e) the direct term $\Lambda\!\cdot\!V$ under a key light at
  $(\text{az},\text{el})=(-30^\circ,30^\circ)$; (f) the percentile-normalised
  shadow map $S$, in which the brightest quartile of skin sits at unity;
  (g) the final gain $g$, bounded below by $0.82$; (h) the composited
  result. Panel (h) differs from (a) only in linear radiance---channel
  ratios are preserved.}
  \label{fig:maps}
\end{figure*}

\subsection{Key-light estimation}
\label{sec:light}

Estimating illumination from a face is normally done \emph{on} the face. That is
unavailable as a \emph{primary} cue: the inner cheek is exactly the region
whose shading we intend to synthesise, so reading light from it would recover
the donor's studio lighting and reinforce the flatness we are trying to remove.

We therefore put most of the weight on cues outside the shading target.
Modelling the torso as a vertical cylinder gives a left/right luminance
asymmetry that maps to azimuth---but only under an approximately uniform
torso albedo; a two-tone shirt or strap is an albedo edge, not a lighting
cue. The centroid of background highlights gives a second azimuth estimate;
the luminance halo immediately around the hair and jaw gives a third and is
the most reliable in tight crops, where a $2.2\times$ face box admits little
torso. Four weaker cues---a spherical fit to hair shading, eye catchlight
position, agreement between feature-region luminance and the depth normals,
and the chin shadow---refine elevation. On a swapped face the last three
live on the \emph{donor} and are circular: catchlights and residual chin
shadow belong to the studio setup, and a flat cheek versus a non-flat mesh
votes for frontal light. We therefore assign them the four lowest weights
(Fig.~\ref{fig:cues}, left). Each cue $i$ yields a direction $\lv_i$ with
confidence $w_i$, and the fused direction is the normalised weighted
resultant
\begin{equation}
  \lv \;=\; \frac{\sum_i w_i \lv_i}{\bigl\lVert \sum_i w_i \lv_i \bigr\rVert}.
  \label{eq:fuse}
\end{equation}
Let $R=\lVert\sum_i w_i\lv_i\rVert/\sum_i w_i$ be the mean resultant length.
Disagreement among cues (small $R$) is mapped to a larger ambient ratio
$\rho$, i.e.\ softer light, which is a conservative fail-safe rather than a
measurement of scene contrast. In the analytic experiments of
Sec.~\ref{sec:exp} we hold $\rho=0.58$ fixed in order to isolate the
geometry-to-gain map.

Two remarks. First, $\lv$ orients shadows only; the optional tint path is
held at zero, otherwise Proposition~\ref{prop:bound}(ii) is void. Second, an
error in $\lv$ misplaces a soft nose shadow rather than painting a colour
cast---but it does \emph{not} match the \emph{strength} of the host key. A
sunlit shoulder and an $18\%$ face shadow can still disagree.

\subsection{The depth shadow map}
\label{sec:shadow}

\paragraph{Direct term.} Skin exhibits pronounced subsurface transport, so a
hard Lambertian terminator is wrong. We use a wrapped Lambertian with wrap
$w=0.28$ and a sharpening exponent $\gamma=1.4$:
\begin{equation}
  \Lambda \;=\;
  \left[\mathrm{clip}\!\left(\frac{\nn\cdot\lv + w}{1+w},\,0,\,1\right)\right]^{\gamma}.
  \label{eq:lambert}
\end{equation}
The wrap carries a little light past $90^\circ$ (Fig.~\ref{fig:lambert}, left);
the exponent restores some crispness that the wrap removes.

\paragraph{Cast shadows.} A Lambertian term alone cannot place a nose shadow on
a cheek, because that cheek's normal still faces the light. We therefore march a
ray from each pixel toward $\lv$ across the depth buffer. In image coordinates
the step is $\hat{\mathbf{d}}=(\ell_x,-\ell_y)/\lVert(\ell_x,-\ell_y)\rVert$
(the sign on $\ell_y$ converts $+Y$-up into row-down). With $D$ increasing
toward the camera, a sample occludes when it is \emph{closer} than the ray,
i.e.\ when $D$ is larger. Writing $t$ in units of crop width, the ray depth
and a one-sided, distance-attenuated visibility are
\begin{equation}
\begin{aligned}
  m &= c_m\,\ell_z\big/\lVert(\ell_x,-\ell_y)\rVert,\\
  \Delta(t) &= D(\mathbf{x}+t\hat{\mathbf{d}}) - D(\mathbf{x}) - mt - b,\\
  V(\mathbf{x}) &= \min_{t\in(0,t_{\max}]}
    \Bigl[1 - \mathrm{clip}(\Delta(t)/\varepsilon,\,0,\,1)\,
          (1-t/t_{\max})\Bigr],
\end{aligned}
\label{eq:vis}
\end{equation}
with $c_m=0.55$, bias $b=0.004$, occlusion scale $\varepsilon=0.03$, and
$t_{\max}=0.22$. $V$ is then clipped to $[0,1]$ and Gaussian-smoothed
($\sigma_V{=}2$~px). The $\min$ is a strongest-along-ray occlusion, not a
first-hit test; negative $\Delta$ (no hit) leaves that sample at $1$. This
is what produces the alar shadow and the brow shadow over the medial eyelid.

\paragraph{Composition and normalisation.} Ambient and direct contributions are
combined as
\begin{equation}
  \Sigma \;=\; \rho A \;+\; (1-\rho)\,\Lambda V \sqrt{A},
\end{equation}
where the $\sqrt{A}$ on the direct path prevents cavities from being doubly
darkened. Finally---and this step is what makes the method safe---we divide by
the $75$th percentile of $\Sigma$ over skin and clip:
\begin{equation}
  S \;=\; \mathrm{clip}\!\left(\frac{\Sigma}{Q_{75}(\Sigma\,|\,\text{skin})},\,0,\,1\right).
  \label{eq:S}
\end{equation}
By construction the \emph{brightest quartile} of the skin ($25\%$ of pixels
with $\Sigma\ge Q_{75}$) now sits at $S=1$. The remaining three-quarters can
still be darkened; the $\tau$-gate of Sec.~\ref{sec:gain} then returns an
additional band with $S\ge\tau$ as $g=1$. (Dividing by $Q_{25}$ would pin
three-quarters of the skin to unity and would contradict the $56.5\%$
modified-pixel rate in Table~\ref{tab:main}.)

\begin{figure*}[t]
  \centering
  \begin{subfigure}{\textwidth}
    \centering
    \includegraphics[width=\textwidth]{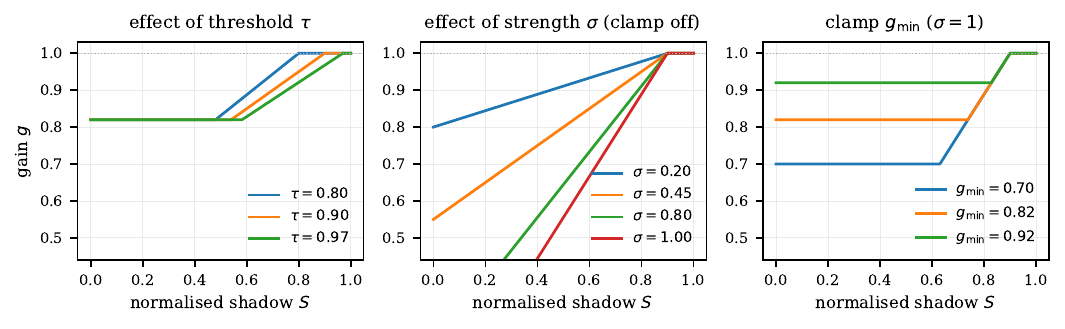}
  \end{subfigure}
  \caption{\textbf{The shadow-to-gain transfer of Eq.~\ref{eq:gain}.} The
  threshold $\tau$ sets the $S$-level treated as fully lit ($g{=}1$);
  raising $\tau$ darkens more of the face. The strength $\sigma$ rescales the
  darkening; the floor $g_{\min}$ truncates the deepest shadows (shown with the clamp disabled in the centre panel to isolate $\sigma$). At the defaults
  the three act together to keep $g$ inside a narrow band near unity.}
  \label{fig:gain}
\end{figure*}

\subsection{From shadow to gain}
\label{sec:gain}

The shadow map is not applied directly. It passes through a three-stage transfer
(Fig.~\ref{fig:gain}) controlled by a threshold $\tau$, a strength $\sigma$, and
a floor $g_{\min}$:
\begin{equation}
  g \;=\; \mathrm{clip}\Bigl(
    1 + \sigma\bigl(\underbrace{\mathrm{clip}(S/\tau,0,1)}_{\text{gated}} - 1\bigr),
    \; g_{\min},\; 1\Bigr).
  \label{eq:gain}
\end{equation}

Each term has a distinct role. The gate $\mathrm{clip}(S/\tau,0,1)$ declares
everything with $S\ge\tau$ to be fully lit, so with the default $\tau=0.90$
well-lit skin is left at $g=1$ \emph{before smoothing}. Raising $\tau$ makes
$S/\tau$ smaller and therefore \emph{enlarges} the darkened set. The strength
$\sigma=0.45$ is the fraction of the gap from $1$ down to $0$ travelled at
$S=0$ (unclamped $g=1-\sigma$); elsewhere the map is the corresponding linear
compression of the gated field, not a physically fitted host contrast. The
floor $g_{\min}=0.82$ caps the worst case at an $18\%$ \emph{linear-radiance}
reduction ($\approx 0.3$~stops; after sRGB encoding the display-referred ratio
is $0.82^{1/2.2}\approx 0.91$). A catastrophic mesh still cannot punch a
black hole into the face; it can still paint an $18\%$ grey patch on the
forehead. The closed-form range of the pre-smoothing operator is
\begin{equation}
  g \in \bigl[\max(g_{\min},\,1-\sigma),\;1\bigr]
      \;=\; [0.82,\,1] \text{ at defaults},
  \label{eq:range}
\end{equation}
which satisfies Proposition~\ref{prop:bound} provided $\tau>0$,
$\sigma\ge 0$, $g_{\min}\in[0,1]$, tint $=0$, and $g$ is clipped again after
smoothing.

Finally $g$ is smoothed at quarter resolution. On the target input---a
\emph{flat} donor---the image luminance $Y(P)$ contains pores and makeup but
\emph{not} the shading edges we just computed, so using $Y$ as a guided-filter
guide would blur those edges and could imprint texture onto $g$. We therefore
guide (or, equivalently, edge-stop) using the shadow map $S$ itself, upsample,
and \emph{re-clip} to $[g_{\min},1]$. Lit-cheek pixels adjacent to a nose
shadow are in general no longer bitwise copies of the input.

\subsection{Face-only compositing}
\label{sec:comp}

The mask $\alpha$ begins as the face oval, eroded slightly and feathered. Near
the hairline a skin gate restricts the effect to skin-chromaticity pixels so
that individual hair strands crossing the forehead are excluded. Eyes and lips
receive partial protection, since sclera and specular lip highlights are
perceptually fragile. The mask is finally multiplied by mesh coverage $C$, which
guarantees $\alpha=0$ wherever no geometry exists. Compositing is
\begin{equation}
  I_{\text{out}} = \tilde I (1-\alpha) + \tilde I g\,\alpha
                 = \tilde I \bigl(1 - \alpha(1-g)\bigr),
  \label{eq:comp}
\end{equation}
which is again a pure per-pixel multiplier bounded in $[g_{\min},1]$---the mask
can only weaken the effect, never extend its range. The maps $g$ and $\alpha$
are resampled from the $512$ crop back to the source face box and pasted into
full-resolution buffers that are $1$ and $0$ respectively outside the crop.
The result is converted to sRGB and written out.

\begin{algorithm}[t]
\small
\caption{Geometry-driven shadow harmonisation}
\label{alg:main}
\begin{algorithmic}[1]
\Require image $I$; $k{=}1.15$, $a{=}0.75$, $\tau{=}0.90$, $\sigma{=}0.45$, $g_{\min}{=}0.82$
\State $\tilde I \gets \mathrm{srgb}^{-1}(I)$
\State box $\gets$ \textsc{Detect}$(I)$; \textbf{if} none \textbf{then return} $I$
\State $P \gets$ crop$(\tilde I,\,s_c{=}2.2)$ resized to $512^2$
\State $\mathcal{V} \gets$ \textsc{FaceMesh}$(P)$; \textbf{if} fail \textbf{then return} $I$
\State $D,\,C \gets$ \textsc{Rasterise}$(\mathcal{V},\,k)$
\State $\nn \gets$ \textsc{Normals}$(D)$;\quad $A \gets$ \textsc{AO}$(D,a)$
\State $\lv,\rho \gets$ \textsc{FuseLightCues}$(\tilde I,$ person matte, hair halo$)$
\State $\Lambda \gets$ \textsc{WrapLambert}$(\nn,\lv)$
\State $V \gets$ \textsc{RayMarchShadow}$(D,\lv)$
\State $\Sigma \gets \rho A + (1-\rho)\Lambda V\sqrt{A}$
\State $S \gets \mathrm{clip}(\Sigma / Q_{75}(\Sigma\,|\,\text{skin}),0,1)$
\State $g \gets \mathrm{clip}(1+\sigma(\mathrm{clip}(S/\tau,0,1)-1),g_{\min},1)$
\State $g \gets$ \textsc{Smooth}$(g,\;\text{guide}{=}\,S)$ at quarter resolution
\State $g \gets \mathrm{clip}(g,\,g_{\min},\,1)$
\State $\alpha \gets$ \textsc{BuildAlpha}(oval, skin gate, eye/lip, $C$)
\State $g,\alpha \gets$ paste into full-resolution maps ($g{=}1$, $\alpha{=}0$ off-crop)
\State \Return $\mathrm{srgb}\bigl(\tilde I\,(1-\alpha(1-g))\bigr)$
\end{algorithmic}
\end{algorithm}

\section{Implementation}

The pipeline runs at a canonical $512\times512$ and is resampled back afterwards,
so cost is independent of source resolution. Rasterisation, ambient occlusion,
and the ray march are the dominant terms; the ray march uses $T{=}48$ steps over
a maximum length of $0.22$ of the crop width, which is sufficient to reach from
the nose tip to the far cheek. The smoother is applied at quarter resolution
and upsampled, as the gain field is intrinsically low-frequency, then clipped.
Batch processing walks an input directory and skips files whose names carry the
pipeline's own suffixes (\texttt{\_harmonized}, \texttt{\_debug},
\texttt{\_compare}). That makes the \emph{tool} idempotent under re-invocation
on its own output directory; the operator itself is not: applying it twice to
the same pixels would darken twice, down to $g_{\min}$.

A nine-tile debug panel is emitted per image---input crop, depth, normals, AO,
shading, shadow map, gain, mask, and result---which in practice is the fastest
way to diagnose a bad output, since each stage is independently inspectable.

\begin{figure*}[t]
  \centering
  \includegraphics[width=\textwidth]{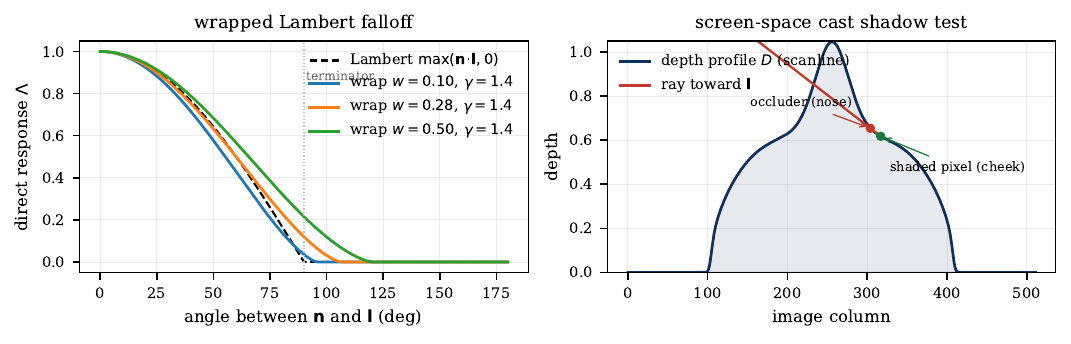}
  \caption{\textbf{Left:} the wrapped Lambertian response. Increasing the wrap
  $w$ moves light past the $90^\circ$ terminator, approximating subsurface
  transport in skin; the exponent $\gamma{=}1.4$ restores falloff crispness.
  \textbf{Right:} the screen-space cast-shadow test on a depth scanline. A ray
  is marched from the shaded pixel toward the light; a sample whose depth
  exceeds the ray (closer to the camera) occludes, and visibility is the
  strongest such occlusion along the ray. This is how a nose places a
  shadow on a cheek whose own normal still faces the light.}
  \label{fig:lambert}
\end{figure*}

\section{Experiments}
\label{sec:exp}

\paragraph{Protocol.} The operator is defined by geometry and a known light
direction, so its transfer can be characterised on a controlled proxy. We
construct an analytic face heightfield---a base cranial spheroid plus
Gaussian nose ridge, alar wings, brow ridges, recessed orbital sockets,
cheekbones, philtrum, lips, and chin crease---paired with a deliberately flat
albedo carrying only mild texture noise. This stands in for the pathological
input the method targets: correct identity, correct colour, no form shadow.
Every number below is measured over face pixels with $\alpha>0.5$ under a
\emph{supplied} key light at azimuth $-30^\circ$, elevation $30^\circ$, with
$\rho=0.58$ held fixed. Light estimation is therefore not part of Table~\ref{tab:main}.
We report statistics of the gain field rather than perceptual scores. No
ground-truth relit pair, no real composite, and no baseline comparison is
claimed: the intent is to characterise the operator's typical magnitude on
this proxy, not to demonstrate perceptual superiority or identity
preservation.

\paragraph{Default operating point.} Table~\ref{tab:main} reports the result. At
$(\tau,\sigma,g_{\min}) = (0.90,0.45,0.82)$, $56.5\%$ of face pixels have
$g<0.995$. The mean multiplier over \emph{all} face pixels is $0.938$ (a $6\%$
mean linear darkening), which is milder than it sounds: the $43.5\%$ of pixels
left near $1$ pull the average up; on modified pixels the mean is $0.890$
($\approx 11\%$ darkening). Only $3.4\%$ of pixels reach the floor, so the
clamp is a safety rail rather than a routine part of the transfer. The
hue-angle histogram in Fig.~\ref{fig:cues}(right) is a sanity check of
Eq.~\ref{eq:core} (maximum shift $<10^{-12}$ degrees, i.e.\ floating-point
noise), not an independent experimental finding. ``Pixels brightened'' is
guaranteed by $g\le 1$; evaluating only $\alpha>0.5$ does not validate the
mask on real hair strands.

\begin{table}[t]
\centering\small
\begin{tabular}{@{}lr@{}}
\toprule
Measurement (face pixels, $\alpha>0.5$) & Value \\
\midrule
Pixels modified ($g<0.995$) & $56.5\%$ \\
Mean gain $\bar g$ (all face pixels) & $0.938$ \\
Mean gain on modified pixels & $0.890$ \\
1st-percentile gain & $0.820$ \\
Pixels at floor $g_{\min}$ & $3.4\%$ \\
Max linear-radiance reduction & $18.0\%$ \\
Max hue-angle shift (corollary of Eq.~\ref{eq:core}) & $<10^{-12}\,^\circ$ \\
Pixels brightened ($g>1$) & $0.0\%$ \\
\bottomrule
\end{tabular}
\caption{Gain statistics at the default configuration, on the analytic
heightfield, with known $\lv$ and fixed $\rho{=}0.58$.}
\label{tab:main}
\end{table}

\paragraph{Threshold.} Figure~\ref{fig:ablation}(a) and Table~\ref{tab:abl} show
that $\tau$ behaves as an exposure-like control on \emph{extent}: raising it
enlarges the set of pixels considered shadowed, monotonically increasing both
coverage and mean darkening. Higher $\tau$ is therefore more aggressive, not
less.

\paragraph{Strength.} Figure~\ref{fig:ablation}(b) shows mean darkening rising
essentially linearly in $\sigma$ until the floor starts binding, after which the
fraction of clamped pixels grows sharply. The default $\sigma=0.45$ sits below
that inflection, which is the intended regime: the clamp should almost never be
the operative constraint. Figure~\ref{fig:sweep} shows the corresponding visual
sweep on the same analytic mesh; $\sigma=0$ recovers the input exactly, and
$\sigma=1$ is visibly over-modelled. Banding around the nasolabial region at
$\sigma=1$ here is smoothness of the \emph{analytic} heightfield; it is
suggestive of, but not a measurement of, MediaPipe mesh smoothness.

\paragraph{Depth scale.} Figure~\ref{fig:ablation}(c) plots the standard
deviation of $g$ against $k$. Over $k\in[0.6,2.0]$ the change is small
($\mathrm{std}(g)$ moves by only a few thousandths), so $k$ is a weak lever
once the wrapped terminator and the $g_{\min}$ clip are in play. We keep the
default $k=1.15$ as a mild anti-flattening of landmark depth, not because the
curve exhibits a sharp saturation knee.

\paragraph{Orientation check.} Figure~\ref{fig:cues}(centre) plots, on this
same heightfield, the left-minus-right mean gain as the supplied azimuth
varies. The asymmetry tracks the sign of the key light, which is the property
the shadow field must have if $\lv$ is correct. It is not a measurement of
the cue fusion error on photographs.

\begin{table}[t]
\centering\small
\begin{tabular}{@{}lrrr@{}}
\toprule
$\tau$ & modified (\%) & $\bar g$ & at floor (\%) \\
\midrule
0.70 & 35.2 & 0.978 & 0.4 \\
0.77 & 41.3 & 0.963 & 1.0 \\
0.85 & 49.2 & 0.948 & 2.1 \\
0.90$^\dagger$ & 56.5 & 0.938 & 3.4 \\
0.93 & 60.6 & 0.933 & 4.0 \\
1.00 & 72.7 & 0.917 & 5.6 \\
\bottomrule
\end{tabular}
\caption{Threshold ablation ($\sigma{=}0.45$, $g_{\min}{=}0.82$).
$\dagger$ denotes the default.}
\label{tab:abl}
\end{table}

\begin{figure*}[t]
  \centering
  \includegraphics[width=\textwidth]{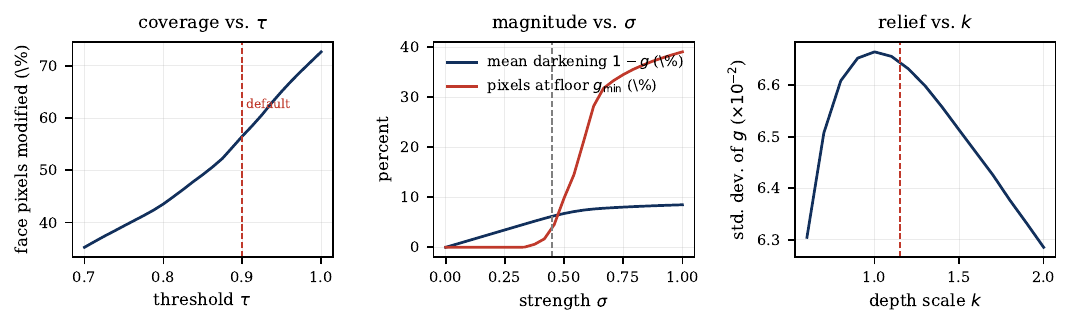}
  \caption{\textbf{Parameter ablations.} (a) Extent of the effect versus the
  threshold $\tau$. (b) Mean darkening and floor-clamping rate versus strength
  $\sigma$; the default sits below the inflection where the clamp begins to
  bind. (c) Relief contrast, measured as the standard deviation of $g$, versus
  depth scale $k$; over this range $\mathrm{std}(g)$ changes only slightly.}
  \label{fig:ablation}
\end{figure*}

\begin{figure*}[t]
  \centering
  \includegraphics[width=\textwidth]{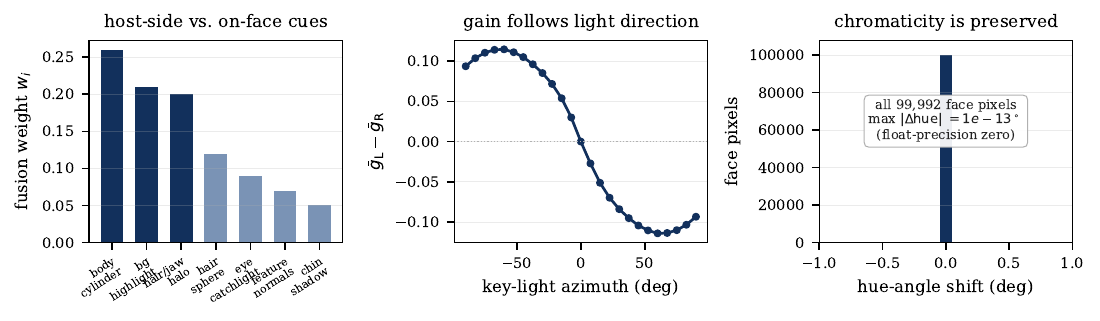}
  \caption{\textbf{Left:} assigned fusion weights. The three strongest cues
  (dark) are host-side; the four weakest (light) lie on or near the face and
  can recover donor lighting on a swap. \textbf{Centre:} left-minus-right
  mean gain on the analytic heightfield as a function of the \emph{supplied}
  key-light azimuth; the sign of the asymmetry tracks the light. This is not
  a photograph-based fusion-error plot. \textbf{Right:} hue-angle shift over
  face pixels---a sanity check of Eq.~\ref{eq:core}, at floating-point
  precision.}
  \label{fig:cues}
\end{figure*}

\section{Discussion}

\paragraph{What the constraints buy.} The multiplicative, darkening-only,
channel-uniform design converts an open-ended synthesis problem into a bounded
one. A detector miss is a no-op. A bad light direction or a bad mesh that still
rasterises produces a misplaced soft shadow of linear magnitude at most
$1-g_{\min}$, not a new identity. That is a different risk profile from
generative relighting, and it is why the method can run unattended---provided
one accepts that ``bounded'' is not the same as ``invisible.''

\paragraph{What the constraints cost.} The method cannot add specular
highlights, cannot correct a face that is already too dark on the wrong side,
cannot brighten a donor to match a brighter host key, and cannot fix a
colour-temperature mismatch (tint is fixed at zero). Magnitude is a user
parameter, not a measurement of body contrast, so a hard-sunlit shoulder can
still make an $18\%$ face shadow look flat. The operator also
\emph{double-shades} any residual donor shading: sockets that already contain
studio AO get darker, and a donor key from the opposite side is not removed.
Landmark meshes under-represent nasolabial folds and eyelid creases, so the
synthesised shadow is systematically smoother than a real one. Screen-space
shadows miss occluders outside the crop. The torso-cylinder cue assumes
roughly uniform clothing albedo.

\paragraph{What the evaluation does not show.} All quantitative tables use one
analytic heightfield and a known $\lv$. They do not measure cue fusion on
photographs, identity cosine before/after, or preference against colour
transfer or neural relighting. Those experiments remain future work; the
present numbers should not be read as a compositing benchmark.

\paragraph{Extensions.} Two directions preserve the safety guarantees. First,
replacing the landmark proxy with a fitted 3DMM or a monocular depth network
would sharpen the relief without changing the operator. Second, a paired
brightening channel with its own independent ceiling $g_{\max}>1$ would admit
gentle rim light while retaining a bounded envelope, at the cost of the
strict ``never brightens'' invariant. A host-contrast estimator for $\sigma$
would address the intensity-matching gap without enlarging the range of $g$.

\section{Conclusion}

We described a face-compositing post-process that injects form shadow without
regenerating skin. Geometry rasterised from dense landmarks yields depth,
normals, and a cavity term; a cue-fused key light orients wrapped Lambertian
shading and screen-space cast shadows; $75$th-percentile normalisation pins
the brightest quartile of skin to unity; and a gated, scaled, clamped
transfer, re-clipped after smoothing, yields a per-pixel multiplier in
$[0.82,1]$ at the defaults. The operator is monotone, chromaticity-preserving,
and bounded \emph{by construction}. On an analytic heightfield the default
configuration modifies just over half the face at a $6\%$ mean darkening
($11\%$ on modified pixels) while rarely engaging the floor. We view this as
evidence that a restricted operator with provable invariants is a reasonable
\emph{conservative} tool for post-hoc compositing---not that it solves lighting
disagreement in general, and not that it has been shown to outperform
unrestricted relighting on real swaps.

\appendix
\onecolumn
\begin{figure}[ht]
  \centering
  \includegraphics[width=\textwidth]{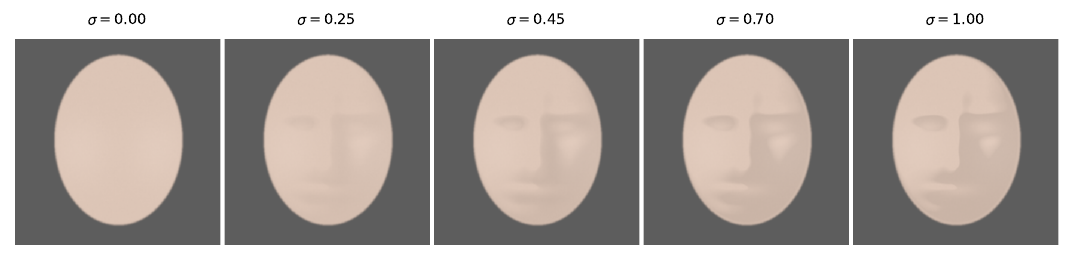}
  \caption{\textbf{Strength sweep} at fixed $\tau{=}0.90$, $g_{\min}{=}0.82$.
  $\sigma{=}0$ returns the input exactly; the default $\sigma{=}0.45$ introduces
  legible orbital and alar form shadow while leaving the lit cheek untouched;
  $\sigma{=}1$ over-applies. Smoothness here is that of the analytic
  heightfield, not a measured MediaPipe mesh.}
  \label{fig:sweep}
\end{figure}

\section*{Appendix A: Default hyperparameters}
\begin{center}\small
\begin{tabular}{@{}llll@{}}
\toprule
Flag & Symbol & Default & Effect \\
\midrule
\texttt{--threshold} & $\tau$ & 0.90 & Gate: $S\ge\tau$ maps to $g{=}1$ before smoothing. Higher $\Rightarrow$ \emph{more} of the face darkens. \\
\texttt{--gain-min} & $g_{\min}$ & 0.82 & Hard floor on the multiplier (linear radiance). Higher $\Rightarrow$ shadows cannot go as dark. \\
\texttt{--strength} & $\sigma$ & 0.45 & Fraction of the $1\!\to\!0$ gap applied at $S{=}0$. Higher $\Rightarrow$ harder application. \\
\texttt{--depth-scale} & $k$ & 1.15 & Exaggeration of mesh relief. Higher $\Rightarrow$ slightly deeper nose/brow shaping. \\
\texttt{--ao-strength} & $a$ & 0.75 & Weight of the depth-unsharp cavity term. Higher $\Rightarrow$ stronger sockets/creases. \\
\texttt{--tint} & --- & 0 & Chromatic adaptation toward the estimated light. Any non-zero value voids Proposition~\ref{prop:bound}(ii). \\
\midrule
crop scale & $s_c$ & 2.2 & Context admitted around the detection box. \\
working size & --- & $512^2$ & Canonical resolution for all geometry and shading. \\
Lambert wrap & $w$ & 0.28 & Light carried past the terminator (subsurface proxy). \\
Lambert sharpness & $\gamma$ & 1.4 & Falloff crispness. \\
ray-march steps & $T$ & 48 & Screen-space cast-shadow samples. \\
ray length & $t_{\max}$ & 0.22 & Maximum march, in crop widths. \\
depth slope & $c_m$ & 0.55 & Converts $\ell_z$ into depth-buffer units in Eq.~\ref{eq:vis}. \\
occlusion scale & $\varepsilon$ & 0.03 & Soft hit threshold in Eq.~\ref{eq:vis}. \\
ray bias & $b$ & 0.004 & Self-occlusion offset. \\
AO sigmoid & $\lambda$ & 2 & Steepness of Eq.~\ref{eq:ao}. \\
AO blur & $\varsigma$ & 14\,px & Gaussian support of $\bar D_\varsigma$. \\
cavity boost & $a_f$ & 0.18 & Semantic crease darkening. \\
analytic $\rho$ & $\rho$ & 0.58 & Ambient ratio used in Sec.~\ref{sec:exp} (not estimated). \\
\bottomrule
\end{tabular}
\end{center}

\section*{Appendix B: Debug panel semantics}
\begin{center}\small
\begin{tabular}{@{}ll@{}}
\toprule
Tile & Interpretation \\
\midrule
input face & the original crop, before any operation \\
face depth map & rasterised $D$: nose forward, orbital sockets recessed \\
surface normals & $\nn$ encoded as RGB; verifies mesh orientation \\
ambient occlusion & $A$: cavity term, independent of light direction \\
current shading & $\Sigma$ before normalisation --- diagnostic only, never applied \\
depth shadow map & $S$ after percentile normalisation: where darkening \emph{could} occur \\
shading gain & $g$: the multiplier actually applied ($1$ = no change) \\
face-only mask & $\alpha$: where the effect is permitted at all \\
harmonized & the composited output \\
\bottomrule
\end{tabular}
\end{center}

\end{document}